\documentclass[sigconf,preprint]{acmart}
\renewcommand\footnotetextcopyrightpermission[1]{} 
\usepackage{amsmath}
\usepackage{amsthm}
\usepackage{cite}
\usepackage{color}
\usepackage{graphicx}
\usepackage{algorithm}
\usepackage{algpseudocode}
\usepackage{url}

\newcommand{\name}{{GLEE}}
\newcommand{\longname}{Geometric Laplacian Eigenmap Embedding}
\DeclareMathOperator*{\argmin}{arg\,min}
\DeclareMathOperator*{\argmax}{arg\,max}

\def\BibTeX{{\rm B\kern-.05em{\sc i\kern-.025em b}\kern-.08emT\kern-.1667em\lower.7ex\hbox{E}\kern-.125emX}}

\begin{document}

\title{GLEE: Geometric Laplacian Eigenmap Embedding}

\author{Leo Torres}
\email{leo@leotrs.com}
\affiliation{
  \institution{Network Science Institute \\
  Northeastern University}
  \city{Boston}
  \state{MA}
  \postcode{02115}
}

\author{Kevin S. Chan}
\email{kevin.s.chan.civ@mail.mil}
\affiliation{%
  \institution{U.S. Army Research Lab}
  \city{Adelphi}
  \state{MD}
  \postcode{20783}
}

\author{Tina Eliassi-Rad}
\email{tina@eliassi.org}
\affiliation{%
  \institution{Network Science Institute \\
  Khoury College of Computer Sciences \\
  Northeastern University}
  \city{Boston}
  \state{MA}
  \postcode{02115}
}

%
\renewcommand{\shortauthors}{Torres et al.}

\begin{abstract}
{Graph embedding seeks to build a low-dimensional representation of a graph $G$. This low-dimensional representation is then used for various downstream tasks. One popular approach is Laplacian Eigenmaps, which constructs a graph embedding based on the spectral properties of the Laplacian matrix of $G$. The intuition behind it, and many other embedding techniques, is that the embedding of a graph must respect node similarity: similar nodes must have embeddings that are close to one another. Here, we dispose of this distance-minimization assumption. Instead, we use the Laplacian matrix to find an embedding with geometric properties instead of spectral ones, by leveraging the so-called simplex geometry of $G$. We introduce a new approach, \emph{\longname{}} (or \name{} for short), and demonstrate that it outperforms various other techniques (including Laplacian Eigenmaps) in the tasks of graph reconstruction and link prediction.}
{Graph embedding, graph Laplacian, simplex geometry.}
\end{abstract}


\maketitle

\section{Introduction}
Graphs are ubiquitous in real-world systems from the internet to the world wide web to social media to the human brain.  The application of machine learning to graphs is a popular and active research area. One way to apply known machine learning methods to graphs is by transforming the graph into a representation that can be directly fed to a general machine learning pipeline. For this purpose, the task of graph representation learning, or graph embedding, seeks to build a vector representation of a graph by assigning to each node a feature vector that can then be fed into any machine learning algorithm.

Popular graph embedding techniques seek an embedding where the distance between the latent representations of two nodes represents their similarity. For example, \citet{chen2018tutorial} calls this the ``community aware'' property (nodes in a community are considered similar, and thus their representations must be close to one another), while \citet{ChenNAKF17} calls it a ``symmetry'' between the node domain and the embedding domain. Others call methods based on this property with various names such as ``positional'' embeddings \citep{ribeiro2019} or ``proximity-based'' embeddings \citep{JinRKKRK19}. Consequently, many of these approaches are formulated in such a way that the distance (in the embedding space) between nodes that are similar (in the original data domain) is small. Here, we present a different approach. Instead of focusing on minimizing the distance between similar nodes, we seek an embedding that preserves the most basic structural property of the graph, namely adjacency; the works \citep{ribeiro2019,JinRKKRK19} call this approach ``structural'' node embeddings. Concretely, if the nodes $i$ and $j$ are neighbors in the graph $G$ with $n$ nodes, we seek $d$-dimensional vectors $s_i$ and $s_j$ such that the adjacency between $i$ and $j$ is encoded in the geometric properties of $s_i$ and $s_j$, for some $d \ll n$. Examples of geometric properties are the dot product of two vectors (which is a measure of the angle between them), the length (or area or volume) of a line segment (or polygon or polyhedron), the center of mass or the convex hull of a set of vectors, among others. In Section~\ref{sec:geom} we propose one such geometric embedding technique, called \longname{} (\name{}), that is based on the properties of the Laplacian matrix of $G$, and we then proceed to compare it to the original formulation of Laplacian Eigenmaps as well as other popular embedding techniques.

\name{} has deep connections with the so-called simplex geometry of the Laplacian \citep{simplex,fiedler}. \citet{fiedler} first made this observation, which highlights the bijective correspondence between the Laplacian matrix of an undirected, weighted graph and a geometric object known as a \textit{simplex}. Using this relationship, we find a graph embedding such that the representations $s_i, s_j$ of two non-adjacent nodes $i$ and $j$ are always orthogonal, $s_i \cdot s_j = 0$, thus achieving a geometric encoding of adjacency. Note that this does not satisfy the ``community aware'' property of \citep{chen2018tutorial}. For example, the geometric embedding $s_i$ of node $i$ will be orthogonal to each non-neighboring node, including those in its community. Thus, $s_i$ is not close to other nodes in its community, whether we define \textit{closeness} in terms of Euclidean distance or  cosine similarity. However, we show that this embedding -- based on the simplex geometry -- contains desirable information, and that it outperforms the original, distance-minimizing, formulation of Laplacian Eigenmaps (LE) on the tasks of graph reconstruction and link prediction in certain cases.

The contributions of this work are as follows. 

\begin{enumerate}
\item We present a geometric framework for graph embedding that departs from the tradition of looking for representations that minimize the distance between similar nodes by highlighting the intrinsic geometric properties of the Laplacian matrix.

\item The proposed method, \longname{} (\name{}), while closely related to the Laplacian Eigenmaps (LE) method, outperforms LE in the tasks of link prediction and graph reconstruction. Moreover, a common critique of LE is that it only considers first-order adjacency in the graph. We show that \name{} takes into account higher order connections (see Section~\ref{sub:prediction}).

\item The performance of existing graph embedding methods (which minimize distance between similar nodes) suffers when the graph's average clustering coefficient is low.  This is not the case for GLEE.
\end{enumerate}

In Section~\ref{sec:background} we recall the original formulation of LE, in order to define the \longname{} (\name{}) in Section~\ref{sec:geom} and discuss its geometric properties. We mention related work in Section~\ref{sec:related} and present experimental studies of \name{} in Section~\ref{sec:exp}. We finish with concluding remarks in Section~\ref{sec:conclusions}.

\section{Background on Laplacian Eigenmaps}\label{sec:background}

Belkin and Niyogi~\citep{BelkinN01,BelkinN03} introduced Laplacian Eigenmaps as a general-purpose method for embedding and clustering an arbitrary data set. Given a data set $\{x_i\}_{i=1}^{n}$, a proximity graph $G=(V,A)$ is constructed with node set $V=\{x_i\}$ and edge weights $\mathbf{A}=(a_{ij})$. The edge weights are built using one of many heuristics that determine which nodes are close to each other and can be binary or real-valued. Some examples are $k$ nearest neighbors, $\epsilon$-neighborhoods, heat kernels, etc. To perform the embedding, one considers the Laplacian matrix of $G$, defined as $\mathbf{L}=\mathbf{D}-\mathbf{A}$, where $\mathbf{D}$ is the diagonal matrix whose entries are the degrees of each node. One of the defining properties of $\mathbf{L}$ is the value of the quadratic form:
\begin{equation}\label{eqn:quadraticform}
y^T \mathbf{L} y = \frac{1}{2} \sum_{i,j} a_{ij}(y_i - y_j)^2.    
\end{equation}
The vector $y^*$ that minimizes the value of (\ref{eqn:quadraticform}) will be such that the total weighted distance between all pairs of nodes is minimized. Here, $y_i$ can be thought of as the one-dimensional embedding of node $i$. One can then extend this procedure to arbitrary $d$-dimensional node embeddings by noting that $tr(\mathbf{Y^T} \mathbf{L} \mathbf{Y}) = \sum_{i,j} a_{ij}\|y_i - y_j\|^2$, where $\mathbf{Y} \in \mathbb{R}^{n \times d}$ and $y_i$ is the $i$th row of $\mathbf{Y}$. The objective function in this case is
\begin{equation}\label{eqn:distmin}
\begin{aligned}
\mathbf{Y^*} = \argmin_{\mathbf{Y} \in \mathbb{R}^{n \times d}} \, &tr(\mathbf{Y^T} \mathbf{L} \mathbf{Y}) \\
\text{s.t.} \,& \mathbf{Y^T} \mathbf{D} \mathbf{Y} = \mathbf{I}
\end{aligned}
\end{equation}
Importantly, the quantity $tr(\mathbf{Y^T} \mathbf{L} \mathbf{Y})$ has a global minimum at $\mathbf{Y}=0$. Therefore, a restriction is necessary to guarantee a non-trivial solution. Belkin and Niyogi~\citep{BelkinN01,BelkinN03} choose $\mathbf{Y^T} \mathbf{D} \mathbf{Y} = \mathbf{I}$, though others are possible. Applying the method of Lagrange multipliers, one can see that the solution of (2) is achieved at the matrix $\mathbf{Y^*}$ whose rows $y_i^*$ are the solutions to the eigenvalue problem
\begin{equation}\label{eqn:distmin-sol}
\mathbf{L} y_i^* = \lambda_i \mathbf{D} y_i^*.
\end{equation}
When the graph contains no isolated nodes, $y_i^*$ is then an eigenvector of the matrix $\mathbf{D^{-1}}\mathbf{L}$, also known as the normalized Laplacian matrix. The embedding of a node $j$ is then the vector whose entries are the $j$th elements of the eigenvectors $y_1^*, y_2^*, ..., y_d^*$.

\section{Proposed Approach: Geometric Laplacian Eigenmaps}\label{sec:geom}
We first give our definition and then proceed to discuss both the algebraic and geometric motivations behind it.

\begin{definition}[GLEE]\label{def:glee}
Given a graph $G$, consider its Laplacian matrix $\mathbf{L}$. Using singular value decomposition we may write $\mathbf{L} = \mathbf{S} \mathbf{S^T}$ for a unique matrix $\mathbf{S}$. Define $\mathbf{S^d}$ as the matrix of the first $d$ columns of $\mathbf{S}$. If $i$ is a node of $G$, define its $d$-dimensional \emph{\longname{}} (\name{}) as the $i$th row of $\mathbf{S^d}$, denoted by $s^d_i$. If the dimension $d$ is unambiguous, we will just write $s_i$.
\end{definition}

\paragraph{Algebraic motivation} In the case of positive semidefinite matrices, such as the Laplacian, the singular values coincide with the eigenvalues. Moreover, it is well known that $\mathbf{S^d}$ is the matrix of rank $d$ that is closest to $\mathbf{L}$ in Frobenius norm, i.e., $\| \mathbf{L} - \mathbf{S^d} \mathbf{(S^d)^T} \|_{F} \le \| \mathbf{L} - \mathbf{M} \|_F$ for all matrices $\mathbf{M}$ of rank $d$. Because of this, we expect $\mathbf{S^d}$ to achieve better performance in the graph reconstruction task than any other $d$-dimensional embedding (see Section~\ref{sub:reconstruction_exp}).

As can be seen from Equation (\ref{eqn:quadraticform}), the original formulation of Laplacian Eigenmaps is due to the fact that the distance between the embeddings of neighboring nodes is minimized, under the restriction $Y^T D Y = I$. We can also formulate \name{} in terms of the distance between neighboring nodes. Perhaps counterintuitively, \name{} solves a distance \emph{maximization} problem, as follows. The proof follows from a routinary application of Lagrange multipliers and is omitted.

\begin{theorem}\label{thm:distminsimplex}
Let $\mathbf{\Lambda}$ be the diagonal matrix whose entries are the eigenvalues of $\mathbf{L}$. Consider the optimization problem
\begin{equation}
\begin{aligned}
\argmax_{\mathbf{Y} \in \mathbb{R}^{n \times d}} \, &tr(\mathbf{Y^T} \mathbf{L} \mathbf{Y}) \\
\text{s.t.} \,& \mathbf{Y^T} \mathbf{Y} = \mathbf{\Lambda}.
\end{aligned}
\end{equation}
Its solution is the matrix $\mathbf{S^d}$ whose columns are the eigenvectors corresponding to the largest eigenvalues of $\mathbf{L}$. If $d=n$ then ${\mathbf{L} = \mathbf{S^d} \mathbf{\left(S^d\right)^T}}$. \qed
\end{theorem}

The importance of Theorem \ref{thm:distminsimplex} is to highlight the fact that distance-minimization may be misleading when it comes to exploiting the properties of the embedding space. Indeed, the original formulation of Laplacian Eigenmaps, while well established in Equation~\ref{eqn:distmin}, yields as result the eigenvectors corresponding to the \emph{lowest} eigenvalues of $\mathbf{L}$. However, standard results in linear algebra tell us that the best low rank approximation of $\mathbf{L}$ is given by the eigenvectors corresponding to the \emph{largest} eigenvalues. Therefore, these are the ones used in the definition of \name{}.

\paragraph{Geometric motivation} The geometric reasons underlying Definition \ref{def:glee} are perhaps more interesting than the algebraic ones. A recent review paper \citep{simplex} highlights the work of \citet{fiedler}, who discovered a bijective correspondence between the Laplacian matrix of a graph and a higher-dimensional geometric object called a \emph{simplex}.
\begin{definition}
Given a set of $k+1$ $k$-dimensional points $\{p_i\}_{i=0}^{k}$, if they are affinely independent (i.e., if the set of $k$ points $\{p_0 - p_i\}_{i=1}^k$ is linearly independent), then their convex hull is called a \emph{simplex}.
\end{definition}
A simplex is a high-dimensional polyhedron that is the generalization of a 2-dimensional triangle or a 3-dimensional tetrahedron. To see the connection between the Laplacian matrix of a graph and simplex geometry we invoke the following result. The interested reader will find the proof in \citep{simplex,fiedler}.
\begin{theorem}\label{thm:simplex}
Let $\mathbf{Q}$ be a positive semidefinite $k \times k$ matrix. There exists a $k \times k$ matrix $\mathbf{S}$ such that $\mathbf{Q} = \mathbf{S} \mathbf{S^T}$. The rows of $\mathbf{S}$ lie at the vertices of a simplex if and only if the rank of $\mathbf{Q}$ is $k-1$. \qed
\end{theorem}

\begin{corollary}\label{cor:lapl}
Let $G$ be a connected graph with $n$ nodes. Its Laplacian matrix $\mathbf{L}$ is positive semidefinite, has rank $n-1$, and has eigendecomposition $\mathbf{L} = \mathbf{P} \mathbf{\Lambda} \mathbf{P^T}$. Write $\mathbf{S} = \mathbf{P} \sqrt{\mathbf{\Lambda}}$. Then, $\mathbf{L} = \mathbf{S} \mathbf{S^T}$ and the rows of $\mathbf{S}$ are the vertices of a $(n-1)$-dimensional simplex called the \emph{simplex of $G$}. \qed
\end{corollary}

Corollary \ref{cor:lapl} is central to the approach in \citep{simplex}, providing a correspondence between graphs and simplices. Corollary~\ref{cor:lapl} also shines a new light on \name{}: \textbf{the matrix $\mathbf{S^d}$ from Definition \ref{def:glee} corresponds to the first $d$ dimensions of the simplex of $G$.} In other words, computing the \name{} embeddings of a graph $G$ is equivalent to computing the simplex of $G$ and projecting it down to $d$ dimensions. We proceed to explore the geometric properties of this simplex that can aid in the interpretation of \name{} embeddings. We can find in \citep{simplex} the following result.

\begin{corollary}\label{cor:angles}
Let $s_i$ be the $i$th row of $\mathbf{S}$ in Corollary~\ref{cor:lapl}. $s_i$ is the simplex vertex corresponding to node $i$, and satisfies $\|s_i\|^2 = \deg(i)$, and $s_i \cdot s_j^T = -a_{ij}$, where $\deg(i)$ is the degree of $i$. In particular, $s_i$ is orthogonal to the embedding of any non-neighboring node $j$. \qed
\end{corollary}

Corollary \ref{cor:angles} highlights some of the basic geometric properties of the simplex (such as lengths and dot products) that can be interpreted in graph theoretical terms (resp., degrees and adjacency). In Figure~\ref{fig:simplex} we show examples of these properties. It is worth noting that other common matrix representations of graphs do not present a spectral decomposition that yields a simplex. For example, the adjacency matrix $\mathbf{A}$ is not in general positive semidefinite, and the normalized Laplacian $\mathbf{D^{-1}}\mathbf{L}$ (used by LE) is not symmetric. Therefore, Theorem~\ref{thm:simplex} does not apply to them. We now proceed to show how to take advantage of the geometry of \name{} embeddings, which can all be thought of as coming from the simplex, in order to perform common graph mining tasks. In the following we focus on unweighted, undirected graphs.

\begin{figure*}[t]
\centering
\includegraphics[width=1\textwidth]{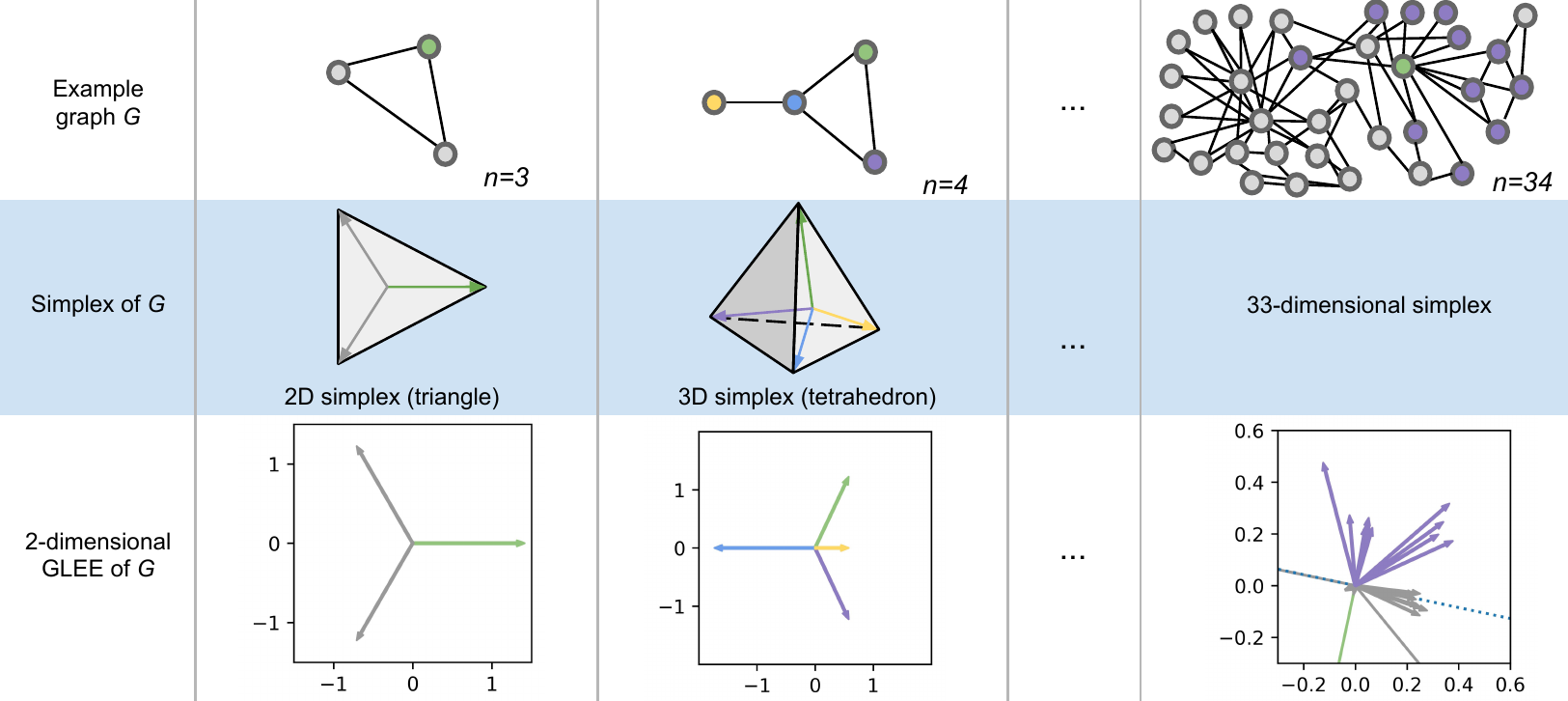}
\caption{\label{fig:simplex}
Simplex geometry and \name{}. Given a graph $G$ with $n$ nodes (top row), there is a $(n-1)$-dimensional simplex that perfectly encodes the structure of $G$, given by the rows of the matrix $S$ from Corollary~\ref{cor:lapl} (middle row). The first $d$ columns of $S$ yield the \longname{} (\name{}) of $G$ (bottom row). In each example, embeddings are color-coded according to the node they represent. For $n=3$, all nodes in the triangle graph are interchangeable. Accordingly, their embeddings all have the same length and subtend equal angles with each other. For $n=4$, the green and purple nodes are interchangeable, and thus their embeddings are symmetric. Note that the length of each embedding corresponds to the degree of the corresponding node. For $n=34$ we show the Karate Club network \citep{zachary}, in which we highlight one node in green and all of its neighbors in purple. In the bottom right panel, the dotted line is orthogonal to the green node's embedding. Note that most of the non-neighbors' embeddings (in gray) are close to orthogonal to the green node's embedding, while all neighbors (in purple) are not.}
\end{figure*}

\subsection{Graph Reconstruction}\label{sub:reconstruction}
For a graph $G$ with $n$ nodes, consider its $d$-dimensional \name{} embedding $\mathbf{S^d}$. When $d=n$, in light of Corollary~\ref{cor:angles}, the dot product between any two embeddings $s_i, s_j$ can only take the values $-1$ or $0$ and one can reconstruct the graph perfectly from its simplex. However, if $d<n$, the distribution of dot products will take on real values around $-1$ and $0$ with varying amounts of noise; the larger the dimension $d$, the less noise we find around the two modes. It is important to distinguish which nodes $i, j$ have embeddings $s_i, s_j$ whose dot product belongs to the mode at $0$ or to the mode at $-1$, for this determines whether or not the nodes are neighbors in the graph. One possibility is to simply ``split the difference'' and consider $i$ and $j$ as neighbors whenever $s_i \cdot s_j < -0.5$. More generally, given a graph $G$ and its embedding $\mathbf{S^d}$, define $\mathbf{\hat{L}(\theta)}$ to be the estimated Laplacian matrix using the above heuristic with threshold $\theta$, that is
\begin{equation}\label{eqn:L-estimate}
\mathbf{\hat{L}_{ij}(\theta)} = \left\{
\begin{array}{ll}
      -1 & s_i \cdot s_j^T < \theta \\
      0 & otherwise. \\
\end{array} 
\right.
\end{equation}
Then, we seek the value of $\theta$, call it $\theta_{\text{opt}}$, that minimizes the loss
\begin{equation}\label{eqn:Lhat-objective}
\theta_{\text{opt}} = \argmin_{\theta \in [-1, 0]} \|\mathbf{L} - \mathbf{\hat{L}}(\theta)\|_F^2.
\end{equation}
If all we have access to is the embedding, but not the original graph, we cannot optimize Equation~(\ref{eqn:Lhat-objective}) directly. Thus, we have to estimate $\theta_{\text{opt}}$ heuristically. As explained above, one simple estimator is the constant $\hat{\theta}_c = -0.5$. We develop two other estimators: $\hat{\theta}_k, \hat{\theta}_g$, obtained by applying Kernel Density Estimation and Gaussian Mixture Models, respectively. We do so in Appendix~\ref{app:estimators} as their development has little to do with the geometry of \name{} embeddings. Our experiments show that different thresholds $\theta_c$, $\theta_k$, and $\theta_g$ produce excellent results on different data sets; see Appendix~\ref{app:estimators} for discussion.

\subsection{Link Prediction}\label{sub:prediction}
Since the objective of \name{} is to directly encode graph structure in a geometric way, rather than solve any one particular task, we are able to use it in two different ways to perform link prediction. These are useful in different kinds of networks.

\subsubsection{Number of Common Neighbors}\label{subsub:common}
It is well known that heuristics such as number of common neighbors (CN) or Jacard similarity (JS) between neighborhoods  are highly effective for the task of link prediction in networks with a strong tendency for triadic closure \citep{SarkarCM11}. Here, we show that we can use the geometric properties of \name{} in order to approximately compute CN. For the purpose of exposition, we assume $d=n$ unless stated otherwise in this section.

Given an arbitrary subset of nodes $V$ in the graph $G$, we denote by $|V|$ its number of elements. We further define \emph{the centroid of $V$}, denoted by $C_V$, as the centroid of the simplex vertices that correspond to its nodes, i.e., $C_V = \frac{1}{|V|} \sum_{i \in V} s_i$. The following lemma, which can be found in \citep{simplex}, highlights the graph-theoretical interpretation of the geometric object $C_V$.

\begin{lemma}[From \citep{simplex}]\label{lem:centroid}
Given a graph $G$ and its \name{} embedding $S$, consider two disjoint node sets $V_1$ and $V_2$. Then, the number of edges with one endpoint in $V_1$ and one endpoint in $V_2$, is given by
\begin{equation}
    -|V_1| |V_2| \,\, C_{V_1}^T \cdot C_{V_2}
\end{equation}
\end{lemma}

\begin{proof}
By linearity of the dot product, we have 
\begin{align}
|V_1||V_2| C_{V_1}^T \cdot C_{V_2} &= \sum_{i\in V_1} \sum_{j\in V_2} s_i \cdot s_j^T = - \sum_{i\in V_1} \sum_{j\in V_2} a_{ij}
\end{align}
The expression on the right is precisely the required quantity.
\end{proof}

Lemma~\ref{lem:centroid} says that we can use the dot product between the centroids of two node sets to count the number of edges that are shared by them. Thus, we now reformulate the problem of finding the number of common neighbors between two nodes in terms of centroids of node sets. In the following, we use $N(i)$ to denote the neighborhood of node $i$, that is, the set of nodes connected to it.

\begin{lemma}\label{lem:common}
Let $i,j \in V$ be non-neighbors. Then, the number of common neighbors of $i$ and $j$, denoted by $CN(i,j)$, is given by
\begin{equation}\label{eqn:cn}
    CN(i,j) = -\deg(i) \, C_{N(i)} \cdot s_j^T = -\deg(j) \,\, C_{N(j)} \cdot s_i^T
\end{equation}
\end{lemma}

\begin{proof}
Apply Lemma~\ref{lem:centroid} to the node sets $V_1=N(i)$ and $V_2=\{j\}$, or, equivalently, to $V_1=N(j)$ and $V_2=\{i\}$.
\end{proof}

Now assume we have the $d$-dimensional \name{} of $G$. We approximate $CN(i,j)$ by estimating both $\deg(i)$ and $C_{N(j)}$. First, we know from Corollary~\ref{cor:angles} that $\deg(i) \approx \|s_i^d\|^2$. Second, we define the approximate neighbor set of $i$ as $\hat{N}(i) = \{k: s_k^d \cdot (s_i^d)^T < \hat{\theta}\}$, where $\hat{\theta}$ is any of the estimators from Section~\ref{sub:reconstruction}. We can now write
\begin{equation}
    CN(i,j) \approx -\|s_i^d\|^2 C_{\hat{N}(i)} \cdot (s_j^d)^T
\end{equation}
The higher the value of this expression, the more confident is our prediction that the link $(i,j)$ exists.

\subsubsection{Number of Paths of Length 3}\label{subsub:l3}
A common critique of the original Laplacian Eigenmaps algorithm is that it only takes into account first order connections, which were considered in Section~\ref{subsub:common}. Furthermore, \citet{Kovacs} point out that the application of link prediction heuristics CN and JS does not have a solid theoretical grounding for certain types of biological networks such as protein-protein interaction networks. They further propose to use the (normalized) number of paths of length three (L3) between two nodes to perform link prediction. We next present a way to approximate L3 using \name{}. This achieves good performance in those networks where CN and JS are invalid, and show that \name{} can take into account higher-order connectivity of the graph.

\begin{lemma}\label{lem:l3}
Assume $S$ is the \name{} of a graph $G$ of dimension $d=n$. Then, the number of paths of length three between two distinct nodes $i$ and $j$ is 
\begin{equation}\label{eqn:l3}
    L3(i,j) = - \deg(i)\deg(j) \, C_{N(i)} \cdot C_{N(j)}^T + \sum_{k\in N(i)\cap N(j)} \|s_k\|^2
\end{equation}
\end{lemma}

\begin{proof}
The number of paths of length three between $i$ and $j$ is $(\mathbf{A^3})_{ij}$, where $\mathbf{A}$ is the adjacency matrix of $G$. We have
\begin{align}
    (\mathbf{A^3})_{ij} &= \sum_{k \in N(i)} \sum\limits_{\substack{l \in N(j)\\ l\neq k}} a_{kl} = -\sum_{k \in N(i)} \sum\limits_{\substack{l \in N(j)\\ l\neq k}} s_k \cdot s_l^T \\
    &= - \sum_{k \in N(i)} \sum_{l \in N(j)} s_k \cdot s_l^T + \sum_{k \in N(i) \cap N(j)} s_k \cdot s_k^T \\
    &= -|N(i)||N(j)| C_{N(i)} \cdot C_{N(j)}^T + \sum_{k \in N(i) \cap N(j)} \|s_k\|^2,
\end{align}
where the last expression follows by the linearity of the dot product, and  is equivalent to (\ref{eqn:l3}).
\end{proof}

When $d<n$, we can estimate $\deg(i)$ by $\|s_i^d\|^2$ and $N(i)$ by $\hat{N}(i)$ as before, with the help of an estimator $\hat{\theta}$ from Section~\ref{sub:reconstruction}.

\subsection{Runtime analysis}\label{sec:bigoh}
On a graph $G$ with $n$ nodes, finding the $k$ largest eigenvalues and eigenvectors of the Laplacian takes $O(kn^2)$ time, if one uses algorithms for fast approximate singular value decomposition \citep{trefethen1997numerical,halko2011finding}. Given a  $k$-dimensional embedding matrix $S$, reconstructing the graph is as fast as computing the product $S \cdot S^T$ and applying the threshold $\theta$ to each entry, thus it takes $O(n^\omega + n^2)$, where $\omega$ is the exponent of matrix multiplication. Approximating the number of common neighbors between nodes $i$ and $j$ depends only on the dot products between embeddings corresponding to their neighbors, thus it takes $O(k \times \min(\deg(i), \deg(j)))$, while approximating the number of paths of length 3 takes $O(k \times \deg(i) \times \deg(j))$.

\section{Related Work}\label{sec:related}
Spectral analyses of the Laplacian matrix have multiple applications in graph theory, network science, and graph mining \citep{newmanbook,spielman,vanmieghembook}. Indeed, the eigendecomposition of the Laplacian has been used for sparsification \citep{spielman2011graph}, clustering \citep{luxburg07}, dynamics \citep{vanmieghem2014,prakash2014}, robustness \citep{jamakovic,shahrivar}, etc. We here discuss those applications that are related to the general topic of this work, namely, dimensionality reduction of graphs.

One popular application is the use of Laplacian eigenvectors for graph drawing \citep{pisanskiS00,koren}, which can be thought of as graph embedding for the specific objective of visualization. In \citep{pisanskiS00} one such method is outlined, which, similarly to \name{}, assigns a vector, or higher-dimensional position, to each node in a graph using the eigenvectors of its Laplacian matrix, in such a way that the resulting vectors have certain desirable geometric properties. However, in the case of \citep{pisanskiS00}, those geometric properties are externally enforced as constraints in an optimization problem, whereas \name{} uses the intrinsic geometry already present in a particular decomposition of the Laplacian. Furthermore, their method focuses on the eigenvectors corresponding to the smallest eigenvalues of the Laplacian, while \name{} uses those corresponding to the largest eigenvalues, i.e. to the best approximation to the Laplacian through singular value decomposition.

On another front, many graph embedding algorithms have been proposed, see for example \citep{GoyalF18,HamiltonYL17} for extensive reviews. Most of these methods fall in one of the following categories: matrix factorization, random walks, or deep architectures. Of special importance to us are methods that rely on matrix factorization. Among many advantages, we have at our disposal the full toolbox of spectral linear algebra to study them \citep{LevinRMP18,charisopoulos19,ChenT15,ChenT17}. Examples in this category are the aforementioned Laplacian Eigenmaps (LE) \citep{BelkinN01,BelkinN03} and Graph Factorization (GF) \citep{AhmedSNJS13}. One important difference between \name{} and LE is that LE uses the small eigenvalues of the normalized Laplacian $\mathbf{D^{-1}} \mathbf{L}$, while \name{} uses the large eigenvalues of $\mathbf{L}$. Furthermore, LE does not present the rich geometry of the simplex. Graph Factorization (GF) finds a decomposition of the weighted adjacency matrix $\mathbf{W}$ with a regularization term. Their objective is to find embeddings $\{s_i\}$ such that $s_i \cdot s_j = a_{ij}$, whereas in our case we try to reconstruct $s_i \cdot s_j = \mathbf{L_{ij}}$. This means that the embeddings found by Graph Factorization will present different geometric properties. There are many other methods of dimensionality reduction on graphs that depend on matrix factorization \citep{KuangPD12,cai2010graph,WangCWP0Y17}. However, even if some parameterization, or special case, of any of these methods results in a method resembling the singular value decomposition of the Laplacian (thus imitating \name{}), to the authors' knowledge none of these methods make direct use of its intrinsic geometry.  

Among the methods based on random walks we find DeepWalk \citep{PerozziAS14} and node2vec \citep{GroverL16}, both of which adapt the framework of word embeddings \citep{mikolov} to graphs by using random walks and optimize a shallow architecture. It is also worth mentioning NetMF \citep{QiuDMLWT18} which unifies several methods in a single algorithm that depends on matrix factorization and thus unifies the two previous categories.

Among the methods using deep architectures, we have the deep autoencoder Structural Deep Network Embedding (SDNE) \citep{WangC016}. It penalizes representations of similar nodes that are far from each other using the same objective as LE. Thus, SDNE is also based on the distance-minimization approach. There is also \citep{CaoLX16} which obtains a non-linear mapping between the probabilistic mutual information matrix (PMI) of a sampled network and the embedding space. This is akin to applying the distance-minimization assumption not to the graph directly but to the PMI matrix.

Others have used geometric approaches to embedding. For example, \citep{EstradaSP14} and \citep{pereda2018machine} find embeddings on the surface of a sphere, while \citep{papadopoulos2012} and \citep{NickelK18} use the hyperbolic plane. These methods are generally developed under the assumption that the embedding space is used to generate the network itself. They are therefore aimed at recovering the generating coordinates, and not, as in \name's case, at finding a general representation suitable for downstream tasks.

\section{Experiments}\label{sec:exp}
We put into practice the procedures detailed in Sections~\ref{sub:reconstruction} and~\ref{sub:prediction} to showcase \name's performance in the tasks of link prediction and graph reconstruction. Code to compute the \name{} embeddings of networks and related computations is publicly available at \citep{glee-code}. For our experiments, we use the following baselines: GF because it is a direct factorization of the adjacency matrix, node2vec because it is regarded as a reference point among those methods based on random walks, SDNE because it aims to recover the adjacency matrix of a graph (a task \name{} excels at), NetMF because it generalizes several other well-known techniques, and LE because it is the method that most directly resembles our own. In this way we cover all of the categories explained in Section~\ref{sec:related} and use either methods that resemble \name{} closely or methods that have been found to generalize other techniques. For node2vec and SDNE we use default parameters. For NetMF we use the spectral approximation with rank $256$. The data sets we use are outlined in Table~\ref{tab:data}. Beside comparing \name{} to the other algorithms, we are interested in how the graph's structure affects performance of each method. This is why we have chosen data sets have similar number of nodes and edges, but different values of average clustering coefficient. Accordingly, we report our results with respect to the average clustering coefficient of each data set and the number of dimensions of the embedding (the only parameter of \name{}). In Appendix~\ref{app:estimator-comparison} we compare the performance of each estimator explained in Section~\ref{sub:reconstruction}. In the following experiments we use $\hat{\theta}_k$ as our estimator for $\theta_{opt}$.

\begin{table}
    \centering
    \begin{tabular}{ccccc}
    \hline
    Name & $n$ & $m$ & $\bar{c}$ & Type\\
    \hline
    \texttt{PPI} \citep{rolland2014proteome} & $4,182$ & $13,343$ & $0.04$ & protein interaction \\
    \texttt{wiki-vote} \citep{LeskovecHK10} & $7,066$ & $100,736$ & $0.14$ & endorsement \\
    \texttt{caida} \citep{LeskovecHK10} & $26,475$ & $53,381$ & $0.21$ & AS Internet \\
    \texttt{CA-HepTh} \citep{LeskovecKF07} & $8,638$ & $24,806$ & $0.48$ & co-authorship \\
    \texttt{CA-GrQc} \citep{LeskovecKF07} & $4,158$ & $13,422$ & $0.56$ & co-authorship \\
    \hline
    \end{tabular}
    \caption{Data sets used in this work (all undirected, unweighted): number of nodes $n$, number of edges $m$, and average clustering coefficient $\bar{c}$ of the largest connected component of each network. AS stands for autonomous systems of the Internet.}
    \label{tab:data}
\end{table}

\subsection{Graph Reconstruction}\label{sub:reconstruction_exp}
\textbf{Given a GLEE matrix $S^d$, how well can we reconstruct the original graph?} This is the task of graph reconstruction. We use as performance metric the \emph{precision at $k$} measure, defined as the precision of the first $k$ reconstructed edges. Note that \textit{precision at k} must always decrease when $k$ grows large, as there will be few correct edges left to reconstruct.

\begin{figure*}[t]
\centering
\includegraphics[width=0.99\textwidth]{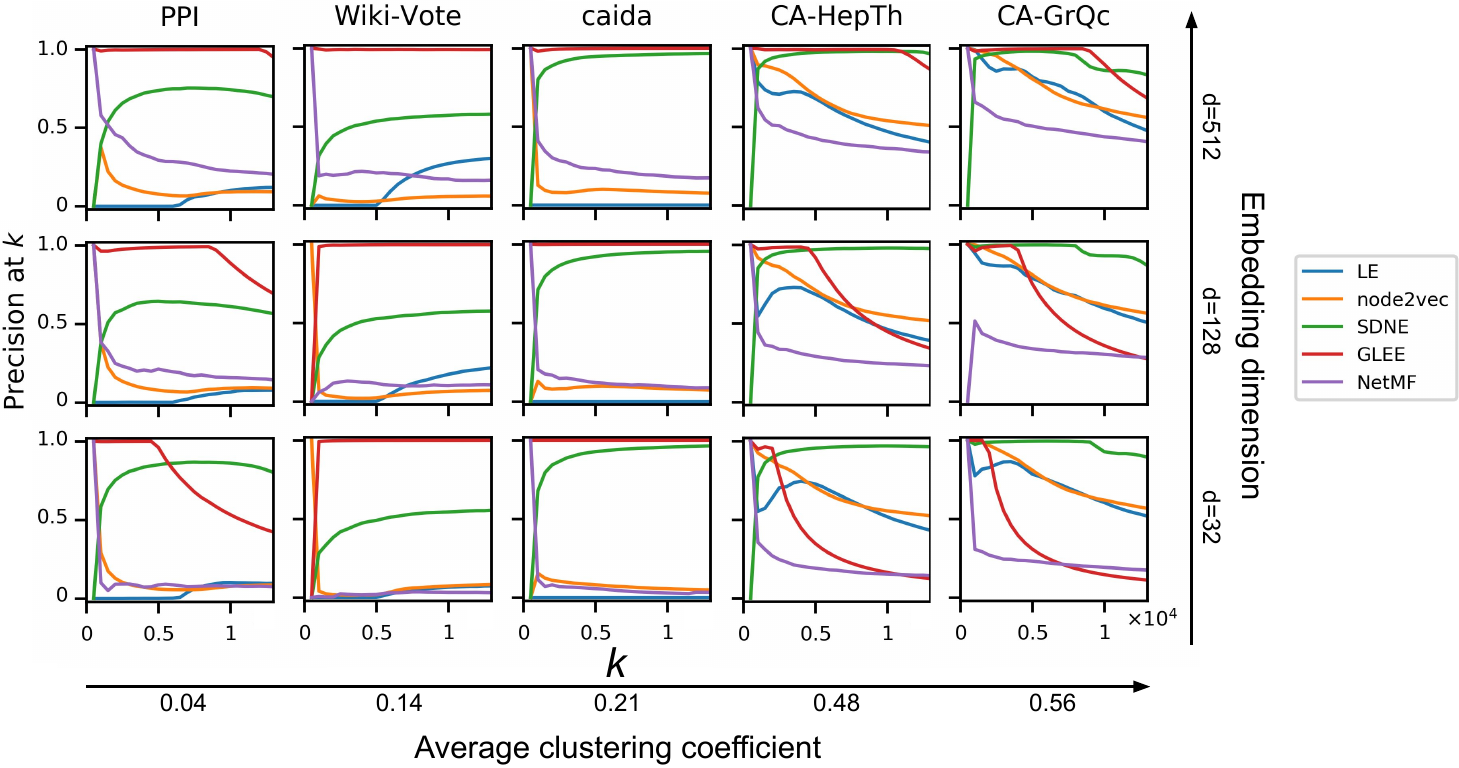}
\caption{\label{fig:reconstruction}
Graph reconstruction. \name{} performs best in networks with low clustering coefficient, presumably because it depends on the large eigenvalues of the Laplacian, which encode micro-level graph structure.}
\end{figure*}

Following Section~\ref{sub:reconstruction}, we reconstruct the edge $(i,j)$ if $s_i^d \cdot s_j^d < \hat{\theta}$. The further the dot product is from $0$ (the ideal value for non-edges), the more confident we are in the existence of this edge. For LE, we reconstruct the edge $(i,j)$ according to how small the distance between their embeddings is. For both GF, node2vec and NetMF, we reconstruct edges based on how high their dot product is. SDNE is a deep autoencoder and thus its very architecture involves a mechanism to reconstruct the adjacency matrix of the input graph.

We show results in Figure~\ref{fig:reconstruction}, where we have ordered data sets from left to right in ascending order of clustering coefficient, and from bottom up in ascending order of embedding dimension. GF results omitted from this Figure as it scored close to $0$ for all values of $k$ and $d$. On \texttt{CA-GrQc}, for low embedding dimension $d=32$, SDNE performs best among all methods, followed by node2vec and LE. However, as $d$ increases, \name{} substantially outperforms all others, reaching an almost perfect precision score at the first 10,000 reconstructed edges. Interestingly, other methods do not substantially improve performance as $d$ increases. This analysis is also valid for \texttt{CA-HepTh}, another data set with high clustering coefficient. However, on \texttt{PPI}, our data set with lowest clustering coefficient, \name{} drastically outperforms all other methods for all values of $d$. Interestingly, LE and node2vec perform well compared to other methods in data sets with high clustering, but their performance drops to near zero on \texttt{PPI}. We hypothesize that this is due to the fact that LE and node2vec depend on the ``community-aware'' assumption, thereby assuming that two proteins in the same cluster would interact with each other. This is the exact point that \citep{Kovacs} refutes. On the other hand, \name{} directly encodes graph structure, making no assumptions about the original graph, and its performance depends more directly on the embedding dimension than on the clustering coefficient, or on any other assumption about graph structure. \name's performance on data sets \texttt{PPI}, \texttt{Wiki-Vote}, and \texttt{caida} point to the excellent potential of our method in the case of low clustering coefficient.

\subsection{Link Prediction}\label{sub:prediction_exp}
\textbf{Given the embedding of a large subgraph of some graph $G$, can we identify which edges are missing?} The experimental setup is as follows. Given a graph $G$ with $n$ nodes, node set $V$ and edge set $E_{obs}$, we randomly split its edges into train and test sets $E_{train}$ and $E_{test}$. We use $|E_{train}|=0.75n$, and we make sure that the subgraph induced by $E_{train}$, denoted by $G_{train}$, is connected and contains every node of $V$. We then proceed to compute the \name{} of $G_{train}$ and test on $E_{test}$. We report AUC metric for this task. We use both techniques described in Sections~\ref{subsub:common} and \ref{subsub:l3}, which we label \name{} and \name-L3 respectively

Figure~\ref{fig:prediction} shows that node2vec repeats the behavior seen in graph reconstruction of increased performance as clustering coefficient increases, though again it is fairly constant with respect to embedding dimension. This observation is also true for NetMF. On the high clustering data sets, LE and \name{} have comparable performance to each other. However, either \name{} or \name-L3 perform better than all others on the low clustering data sets \texttt{PPI}, \texttt{Wiki-Vote}, as expected. Also as expected, the performance of \name-L3 decreases as average clustering increases. Note that \name{} and LE generally improve performance when $d$ increases, whereas node2vec and SDNE do not improve. (GF and SDNE not shown in Figure~\ref{fig:prediction} for clarity. They scored close to $0.5$ and $0.6$ in all data sets independently of $d$.) The reason why none of the methods studied here perform better than $0.6$ AUC in the \texttt{caida} data set is an open question left for future research. We conclude that the hybrid approach of NetMF is ideal for high clustering coefficient, whereas \name{} is a viable option in the case of low clustering coefficient as evidenced by the results on \texttt{PPI}, \texttt{Wiki-Vote}, and \texttt{caida}.

\begin{figure*}[t]
\centering
\includegraphics[width=.85\textwidth]{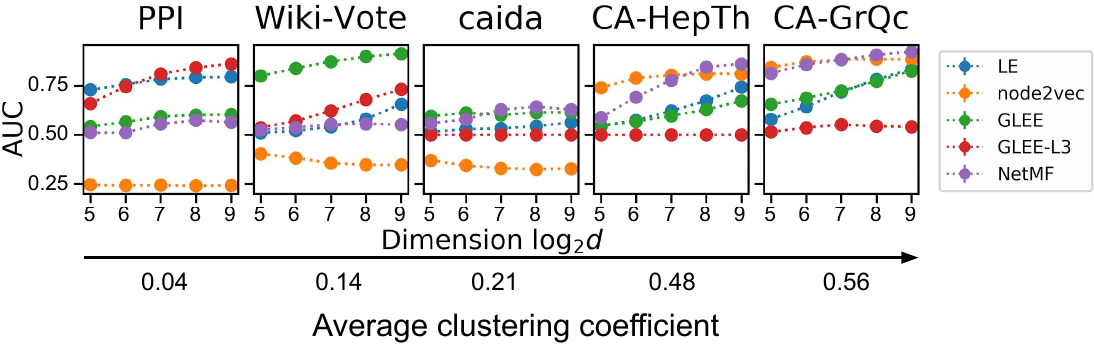}
\caption{\label{fig:prediction}
Link prediction results. We approximate the number of common neighbors (\name{}), or the number of paths of length~3 (\name-L3). Each circle is the average of $10$ realizations; error bars too small to show at this scale. GF and SDNE perform close to $0.5$ and $0.6$ independently of $d$ or data set (not shown). Data sets ordered from left to right in increasing order of clustering.}
\end{figure*}

\section{Conclusions}\label{sec:conclusions}
In this work we have presented the \longname{} (\name{}), a geometric approach to graph embedding that exploits the intrinsic geometry of the Laplacian. When compared to other methods, we find that \name{} performs the best when the underlying graph has low clustering coefficient, while still performing comparably to other state-of-the-art methods when the clustering coefficient is high. We hypothesize that this is due to the fact that the large eigenvalues of the Laplacian correspond to the small eigenvalues of the adjacency matrix and thus represent the structure of the graph at a micro level. Furthermore, we find that \name{}'s performance increases as the embedding dimension increases, something we do not see in other methods. In contrast to techniques based on neural networks, which have many hyperparameters and costly training phases, \name{} has only one parameter other than the embedding dimension, the threshold $\theta$, and we have provided three different ways of optimizing for it. Indeed, \name{} only depends on the SVD of the Laplacian matrix. 

We attribute these desirable properties of \name{} to the fact that it departs from the traditional literature of graph embedding by replacing the ``community aware'' notion (similar nodes' embeddings must be similar) with the notion of directly encoding graph structure using the geometry of the embedding space. In all, we find that GLEE is a promising alternative for graph embedding due to its simplicity in both theoretical background and computational implementation, especially in the case of low clustering coefficient. By taking a direct geometric encoding of graph structure using the simplex geometry, GLEE covers the gap left open by the ``community aware'' assumption of other embedding techniques, which requires high clustering. Future lines of work will explore what other geometric properties of the embedding space can yield interesting insight, as well as what are the important structural properties of graphs, such as clustering coefficient, that affect the performance of these methods.

\section*{Funding}
This work was supported by the National Science Foundation [IIS-1741197]; and by the Combat Capabilities Development Command Army Research Laboratory and was accomplished under Cooperative Agreement Number W911NF-13-2-0045 (U.S. Army Research Lab Cyber Security CRA). The views and conclusions contained in this document are those of the authors and should not be interpreted as representing the official policies, either expressed or implied, of the Combat Capabilities Development Command Army Research Laboratory or the U.S. Government. The U.S. Government is authorized to reproduce and distribute reprints for Government purposes not withstanding any copyright notation here on.

\appendix

\section{Threshold Estimators}\label{app:estimators}
We present two other estimators of $\theta_{opt}$ to accompany the heuristic $\hat{\theta}_c = -0.5$ mentioned in Section~\ref{sub:reconstruction}.

\subsection{Kernel Density Estimation}\label{subsub:kdd}
As can be seen in Figure \ref{fig:hist_glee}, the problem of finding a value of $\theta$ that sufficiently separates the peaks corresponding to edges (around the peak centered at $-1$) and non-edges (around the peak centered at $0$) can be stated in terms of density estimation. That is, given the histogram of values of $s_i \cdot s_j^T$ for all $i,j$, we can approximate the density of this empirical distribution by some density function $f_k$. A good heuristic estimator of $\theta_{\text{opt}}$ is the value that minimizes $f_k$ between the peaks near $-1$ and $0$.  For this purpose, we use Kernel Density Estimation over the distribution of $s_i \cdot s_j^T$ and a box kernel (a.k.a. "top hat" kernel) function to define
\begin{equation}\label{eqn:kde}
f_k(x) \propto \sum_{i<j}^n 1\{x - s_i \cdot s_j^T < h\},
\end{equation}
We then use gradient descent to find the minimal value of $f_k$ between the values of $-1$ and $0$. We call this value $\hat{\theta}_k$. We have found experimentally that a value of $h=0.3$ gives excellent results, achieving near zero error in the reconstruction task (Figure~\ref{fig:hist_glee}, middle row).

\begin{figure}[ht]
\centering
\includegraphics[width=0.99\columnwidth]{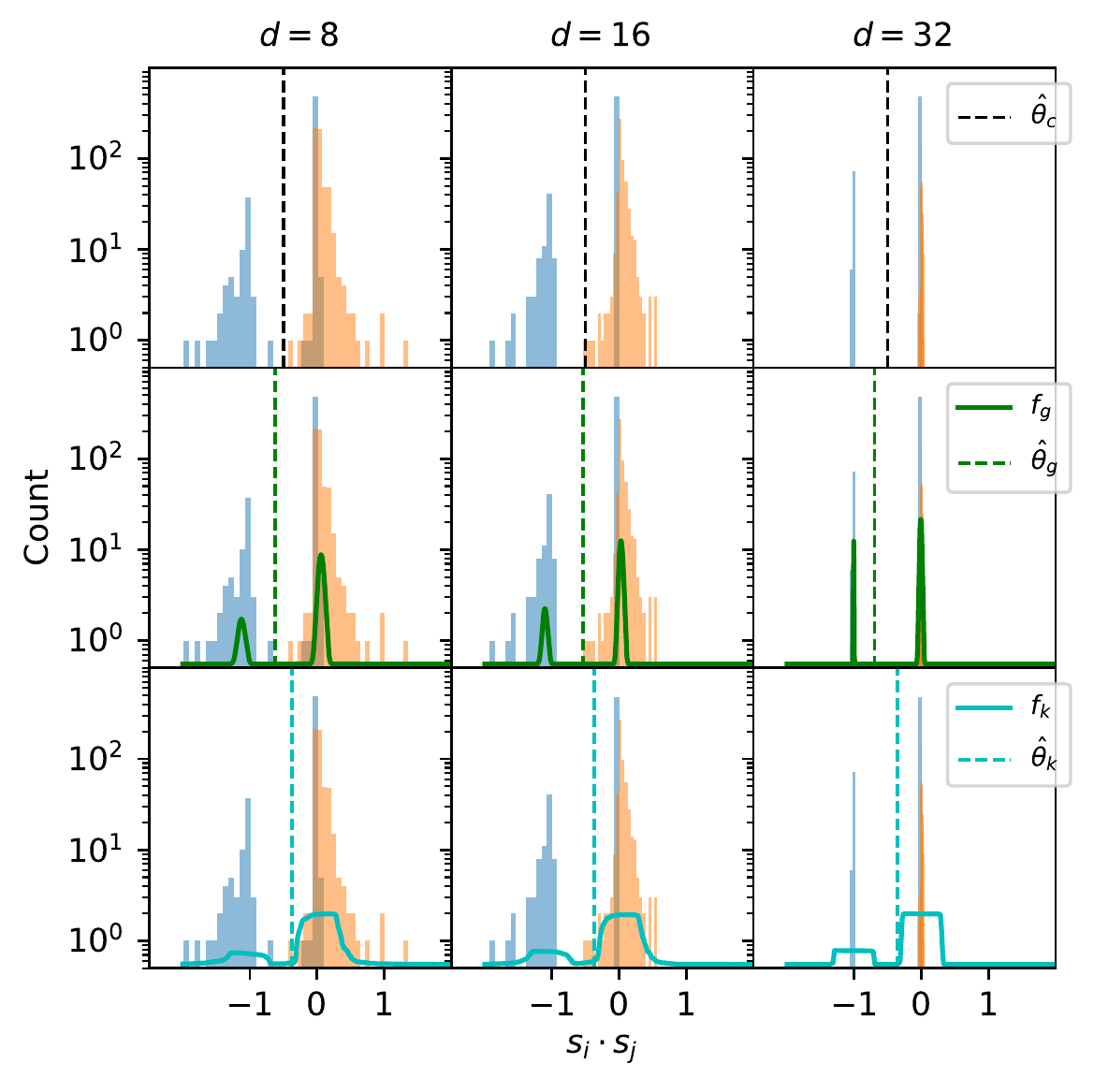}
\caption{\label{fig:hist_glee}
Distribution of dot products $\{s_i \cdot s_j^T\}$ in the Karate Club graph. Columns show different dimension $d$. Dot products of existing edges are in blue. Non-edges are in orange. Each row shows a different estimator of $\theta_{\text{opt}}$. Top row shows the constant value $\hat{\theta}_c = -0.5$. Middle row shows our Gaussian Mixture Model estimator $\hat{\theta}_g$ (Algorithm~\ref{alg:theta-gmm}). Bottom row shows our Kernel Density Estimator  $\hat{\theta}_k$ (Section~\ref{subsub:kdd}).
}
\end{figure}

\subsection{Gaussian Mixture Models}\label{subsub:gmm}
Here we use a Gaussian Mixture Model (GMM) over the distribution of $s_i \cdot s_j$. The model will find the two peaks near $-1$ and $0$ and fit each to a Gaussian distribution. Once the densities of said Gaussians have been found, say $f_1$ and $f_2$, we define the estimator $\hat{\theta}_g$ as that point at which the densities are equal (see Figure~\ref{fig:hist_glee}, bottom row).

However, we found that a direct application of this method yields poor results due to the sparsity of network data sets. High sparsity implies that the peak at $0$ is orders of magnitude higher than the one at $-1$. Thus, the left peak will usually be hidden by the tail of the right one so that the GMM cannot detect it. To solve this issue we take two steps. First, we use a Bayesian version of GMM that accepts priors for the Gaussian means and other parameters. This guides the GMM optimization algorithm to find the right peaks at the right places. Second, we sub-sample the distribution of dot products in order to minimize the difference between the peaks, and then to fix it back after the fit. Concretely, put $r = \sum_{i<j} 1\{s_i \cdot s_j^T < \hat{\theta}_c\}$.  That is, $r$ is the number of dot products less than the constant $\hat{\theta}_c = -0.5$. Instead of fitting the GMM to all the observed dot products, we fit it to the set of all $r$ dot products less than $\hat{\theta}_c$ plus a random sample of $r$ dot products larger than $\hat{\theta}_c$. This temporarily fixes the class imbalance, which we recover after the model has been fit as follows. The GMM fit will yield a density for the sub-sample as $f_g = w_1 f_1 + w_2 f_2$, where $f_i$ is the density of the $i$th Gaussian, and $w_i$ are the mixture weights, for $i=1,2$. Since we sub-sampled the distribution, we will get $w_1 \approx w_2 \approx 0.5$, but we need the weights to reflect the original class imbalance. For this purpose, we define $\hat{w}_1 = \hat{m} / \binom{n}{2}$ and $\hat{w}_2 = 1 - \hat{w}_1$, where $\hat{m}$ is an estimate for the number of edges in the graph. (This can be estimated in a number of ways, for example one may put $\hat{m} = r$, or $\hat{m} = n \log(n)$.) Finally, we define the estimator as the value that satisfies
\begin{equation}\label{eqn:theta-gmm}
\hat{w_1} f_1(\hat{\theta}_g) = \hat{w_2} f_2(\hat{\theta}_g),
\end{equation}
under the constraint that $-1 < \hat{\theta}_g <0$. Since $f_1$ and $f_2$ are known Gaussian densities, Equation~\ref{eqn:theta-gmm} can be solved analytically.

In this case, due to sparsity, the problem of optimizing the GMM is one of non-parametric density estimation with extreme class imbalance. We solve it by utilizing priors for the optimization algorithm, as well as sub-sampling the distribution of dot products, according to some of its known features (i.e., the fact that the peaks will be found near $-1$ and $0$), and we account for the class imbalance by estimating graph sparsity separately. Finally, we define the estimator $\hat{\theta}_g$ according to Equation~\ref{eqn:theta-gmm}. Algorithm~\ref{alg:theta-gmm} gives an overview of this procedure. For a comparison between the effectiveness of the three different estimators $\hat{\theta}_c, \hat{\theta}_k, \hat{\theta}_g$, see Appendix~\ref{app:estimator-comparison}.

\begin{algorithm}
\caption{Estimating $\theta_{opt}$ with a Gaussian Mixture Model}
\label{alg:theta-gmm}
\begin{algorithmic}[1]
    \Procedure{GMM}{$\{s_i\}_{i=1}^n, \hat{\theta}_c, \hat{m}$}
    \State $L \gets \{s_i \cdot s_j^T: s_i \cdot s_j^T < \hat{\theta}_c\}$
    \State $R \gets$ random sample of size $|L|$ of $\{s_i \cdot s_j^T: s_i \cdot s_j^T \geq \hat{\theta}_c\}$
    \State $w_1, w_2, f_1, f_2 \gets$ fit a Bayesian GMM to $L \cup R$
    \State $\hat{w}_1 \gets \hat{m} / \binom{n}{2}$
    \State $\hat{w}_2 \gets 1 - \hat{w_1}$
    \State $\hat{\theta}_g \gets$ solution of $\hat{w}_1 f_1(\theta) = \hat{w}_2 f_2(\theta)$
    \State \textbf{return} $\hat{\theta}_g$
    \EndProcedure
\end{algorithmic}
\end{algorithm}

\section{Estimator Comparison}\label{app:estimator-comparison}
In Section~\ref{sub:reconstruction} and Appendix~\ref{app:estimators} we outlined three different schemes to estimate $\theta_{\text{opt}}$ which resulted in $\hat{\theta}_c, \hat{\theta}_k, \hat{\theta}_g$. \textbf{Which one is the best?} We test each each of these estimators on three random graph models: Erd\"os-R\'enyi (ER) \citep{erdos1960p}, Barab\'asi-Albert (BA) \citep{barabasi1999emergence}, and Hyperbolic Graphs (HG) \citep{krioukov2010hyperbolic}. For each random graph with adjacency matrix $\mathbf{A}$, we compute the Frobenius norm of the difference between the reconstructed  adjacency matrix $\mathbf{\hat{A}}$ using each of the three estimators. In Figure~\ref{fig:estimators} we show our results. We see that $\hat{\theta}_c$ and $\hat{\theta}_k$ achieve similar performance across data sets, while $\hat{\theta}_g$ outperforms the other two for ER at $d=512$, though it has high variability in the other models. From these results we conclude that at low dimensions $d=32$, too much information has been lost and thus there is no hope to learn a value of $\hat{\theta}$ that outperforms the heuristic $\hat{\theta}_c = -0.5$. However, at larger dimensions, the estimators $\hat{\theta}_g$ and $\hat{\theta}_k$ perform better, with different degrees of variability. We conclude also that no single heuristic for $\hat{\theta}$ is best for all types of graphs. In the rest of our experiments we use $\hat{\theta}_k$ as our estimator for $\theta_{opt}$. We highlight that even though $\theta_k$ is better than $\theta_c$ in some data sets, it might be costly to compute, while $\theta_c$ incurs no additional costs.

\begin{figure}[ht]
\centering
\includegraphics[width=0.99\columnwidth]{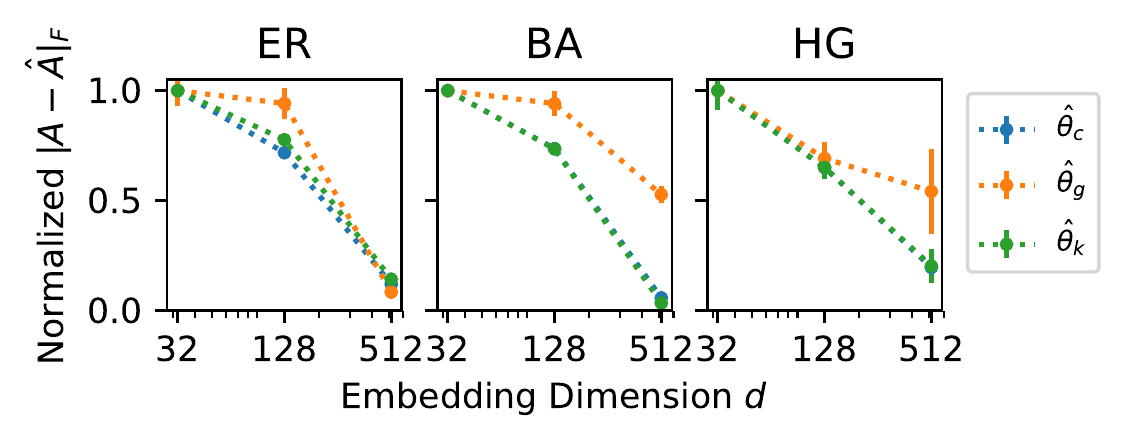}
\caption{\label{fig:estimators}
Estimator comparison. We compute the three different estimators on three different random graph models: Erd\"os-R\'enyi (ER), Barab\'asi-Albert (BA), and Hyperbolic Graphs (HG). All graphs have $n=10^3$ nodes, and average degree $\langle k \rangle \approx 8$. Hyperbolic graphs generated with degree distribution exponent $\gamma = 2.3$. We show the average of $20$ experiments; error bars mark two standard deviations. Values normalized in the range $[0, 1]$.}
\end{figure}

\bibliographystyle{ACM-Reference-Format}
\bibliography{references.bib}

\end{document}